\documentclass[letterpaper, 10pt, conference]{ieeeconf}      

\IEEEoverridecommandlockouts                              

\usepackage[bookmarks=true,bookmarksnumbered=true,bookmarkstype=toc, %
pdfauthor={Liao Wu},%
pdftitle={POEDH},%
pagebackref=true, 
colorlinks, 
citecolor=magenta, 
linkcolor=blue,
anchorcolor=red]{hyperref}

\usepackage{graphicx} 
\usepackage{subfigure}
\usepackage{epsfig} 
\usepackage{mathptmx} 
\usepackage{times} 
\usepackage{amsmath} 

\usepackage{amsthm}
\usepackage{amssymb}  
\usepackage{booktabs} 
\usepackage{threeparttable}
\usepackage{multirow}
\usepackage{cite}
\usepackage{bm}
\usepackage{mathtools}
\usepackage{color}
\usepackage{hyperref}

\title{\LARGE \bf
	Geometric interpretation of the general POE model for a serial-link robot via conversion into D-H parameterization
}

\author{Liao~Wu,~\IEEEmembership{Member,~IEEE}, Ross~Crawford, and Jonathan~Roberts,~\IEEEmembership{Senior Member,~IEEE}
	\thanks{This work was supported by the Vice-Chancellor's Research Fellowship (322450-0096/08) and the IHBI Early Career Researcher Development Scheme (243218-0233/07) awarded to Dr. Liao Wu by Queensland University of Technology.
	}
	\thanks{L. Wu was with the Australian Centre for Robotic Vision, Science and Engineering Faculty, Queensland University of Technology. He is now with the School of Mechanical and Manufacturing Engineering, Faculty of Engineering, University of New South Wales. {\tt\small dr.liao.wu@ieee.org}}
	\thanks{R. Crawford and J. Roberts are with the Australian Centre for Robotic Vision, Science and Engineering Faculty, Queensland University of Technology.}
	\thanks{Source code in MATLAB can be found at \href{https://github.com/drliaowu/GeneralPOE2DH}{https://github.com/drliaowu/GeneralPOE2DH}}
}

\begin{document}
	
\newtheorem{lemma}{Lemma}
\newtheorem{remark}{Remark}

\maketitle

\begin{abstract}
While Product of Exponentials (POE) formula has been gaining maturity in modeling the kinematics of a serial-link robot, the Denavit-Hartenberg (D-H) notation is still the most widely used due to its intuitive and concise geometric interpretation of the robot. This paper has developed an analytical solution to automatically convert a POE model into a D-H model for a robot with revolute, prismatic, and helical joints, which are the complete set of three basic one degree of freedom lower pair joints for constructing a serial-link robot. The conversion algorithm developed can be used in applications such as calibration where it is necessary to convert the D-H model to the POE model for identification and then back to the D-H model for compensation. The equivalence of the two models proved in this paper also benefits the analysis of the identifiability of the kinematic parameters. It is found that the maximum number of identifiable parameters in a general POE model is $5h+4r+2t +n+6$ where $h$, $r$, $t$, and $n$ stand for the number of helical, revolute, prismatic, and general joints, respectively. It is also suggested that the identifiability of the base frame and the tool frame in the D-H model is restricted rather than the arbitrary six parameters as assumed previously.
\end{abstract}


\section{Introduction}


The product of exponentials (POE) formula has been gaining maturity in modeling the kinematics of serial-link robots since it was proposed in 1984 \cite{brockett1984robotic,murray1994mathematical,selig2005geometric,lynch2017modern}.
The method defines an initial state where all the joints are in their zero positions, attaches two frames to the base link and the tool link respectively, and extracts a set of joint twists from the geometry of the joint axes.
Then the kinematics of the robot can be elegantly expressed as a product of exponentials of the joint twists post-multiplied by the transform between the two initial frames.
Some of the advantages of using the POE formula include the simplicity of modeling, the close connection with the advanced Lie groups theory, and the natural avoidance of singularity for calibration \cite{okamura1996kinematic,he2010kinematic,yang2014minimal,wu2015minimal}.

On the other hand, the Denavit-Hartenberg (D-H) notation that was invented in the 1950s \cite{denavit1955kinematic} is still the most widely used method for kinematics modeling \cite{craig2005introduction,siciliano2008springer,siciliano2010robotics}.
The conventions use four D-H parameters to formulate the transformation between adjacent frames that are specially attached to each link and represent the kinematics as a concatenation of the transformations.
Based on the D-H notation, a large body of algorithms and code implementations have been established for the dynamics, motion planning, control, etc. \cite{corke2011robotics}.
Another reason behind the popularity of the D-H model is probably its intuitiveness and conciseness in interpreting the geometry of a serial-link robot.

Due to the complementary advantages of both methods, there is sometimes a need for conversion between them.
Take calibration for example.
The POE model can naturally avoid the singularity problem that is encountered by the D-H model for robots with adjacent parallel joint axes.
However, most commercial robots are internally modeled with the D-H notation in their controllers.
Therefore, to calibrate the robot, it may be necessary to convert the D-H model to the POE model, calibrate the POE model, and then convert the POE model back to the D-H model for compensation.
In our previous work \cite{wu2017an}, we have shown the equivalence between the two models for robots with revolute and prismatic joints and derived an algorithm for the automatic conversion between them.
In this work, we extend the algorithm to cover helical joints, the third and the last type of one degree of freedom (DoF) lower pair joints.

The reason for extending the algorithm for helical joints is twofold.
First, although revolute and prismatic joints have been the two main joint types for the construction of serial-link robots due to their easiness of actuation and adaptability to general tasks, recent progresses in direct-drive helical motors \cite{smadi2012development} and the trend of using robots for specialized tasks such as screw insertion in spinal interventions \cite{bertelsen2013review} have encouraged the adoption of helical joints in robot design and construction.
Having the algorithm cover helical joints will complete the theory for the three basic one DoF lower pair joints.
Second, due to manufacturing and assembly errors, some revolute or prismatic joints do present a pattern of helical motion \cite{okamura1996kinematic}.
Having the algorithm cover helical joints will enable a calibration model to capture such a characteristic and make the kinematic model more accurate after compensation.

Moreover, we will show that the equivalence between the POE model and the D-H model for the three types of joints can be used for the identifiability analysis of the kinematic parameters of a serial-link robot, which is important in calibration and has aroused some debates recently \cite{li2016poe,chen2014determination,yang2014minimal,he2010kinematic}.
With the help of the conversion between the two models, we will show that the maximum number of identifiable parameters in a POE model is $5h+4r+2t+n+6$ where $h$, $r$, $t$, and $n$ stand for the number of helical, revolute, prismatic, and general joints, respectively. 
We will also show the identifiability of the base frame and the tool frame in the D-H model is restricted rather than the arbitrary six parameters as assumed previously \cite{mooring1991fundamentals}.

\section{The General POE Model}

\subsection{Special Euclidean Group $SE(3)$} \label{sec:special}

The set of rigid body motions forms a Lie group called the \textit{Special Euclidean Group} $SE(3)$, whose element can be represented by a homogeneous transformation
$\bm{H}=\left[\begin{smallmatrix}
\bm{R}&\bm{t}\\
\bm{0}^T&1
\end{smallmatrix}\right] \in SE(3)$. 
Here $\bm{R}$ is a 3D rotation belonging to the \textit{Special Orthogonal Group} $SO(3)$ and $\bm{t} \in \mathbb{R}^3$ is a 3D translation.
According to Chasles's theorem, any rigid-body motion in 3D space can be described by a helical motion, i.e., a rotation about a line together with a translation along the line.
This helical motion can be captured by a twist $\bm{\hat{\xi}}=\left[\begin{smallmatrix}
\bm{\hat{\omega}}& \bm{v}\\
\bm{0}^T&0
\end{smallmatrix}\right] \in se(3)$ or its coordinate form $\bm{\xi}=\left[\begin{smallmatrix}
\bm{\omega}^T & \bm{v}^T
\end{smallmatrix}\right]^T \in \mathbb{R}^6$.
Here $\bm{\hat{\omega}}=\left[\begin{smallmatrix}
0 & -\omega_3 & \omega_2\\
\omega_3 & 0 & -\omega_1\\
-\omega_2 & \omega_1 & 0
\end{smallmatrix}\right] \in so(3)$ or its coordinate form $\bm{\omega}=\left[\begin{smallmatrix}
\omega_1 & \omega_2 & \omega_3
\end{smallmatrix}\right]^T$ relates to the rotation, while $\bm{v}=\left[\begin{smallmatrix}
v_1 & v_2 & v_3
\end{smallmatrix}\right]^T$ is determined by both the rotation and the translation.
$se(3)$ and $so(3)$ are the Lie algebras associated with $SE(3)$ and $SO(3)$, respectively.

Specifically, $\bm{\xi}$ can be written in the form as
\begin{equation}\label{eq:xi}
	\bm{\xi}=\left[\begin{matrix}
	\bm{\omega}\\
	\bm{v}
	\end{matrix}\right]	= \left[\begin{matrix}
	\bm{\bar{\omega}}\\
	\bm{\bar{v}}
	\end{matrix}\right]q = \left[\begin{matrix}
	\bm{\bar{\omega}}\\
	\bm{p}\times \bm{\bar{\omega}}+h\bm{\bar{\omega}}
	\end{matrix}\right]q :=\bm{\bar{\xi}}q 
\end{equation} 
where $\bm{\bar{\xi}}$ is the \textit{normalized twist}. 
If $\bm{\omega}\neq\bm{0}$, then $q=||\bm{\omega}||$ is the rotation angle and $\bm{\bar{\omega}}=\bm{\omega}/q$ is the unit vector along the rotation axis.
$\bm{\bar{v}}=\bm{v}/q$ is composed of two components, $\bm{p}\times \bm{\bar{\omega}}$ and $h\bm{\bar{\omega}}$, in which the first item determines where the rotation axis is as $\bm{p}$ is the 3D position of a point on the rotation axis, and the second item determines the translation distance along the rotation axis as $h$ represents the ratio between the translation and the rotation (we may call it the pitch of the helical motion, although it differs from the real pitch by a factor of $2 \pi$).
A pure rotation corresponds to a zero pitch $h=0$.
If $\bm{\omega} = \bm{0}$, which means the motion is a pure translation, $q=||\bm{v}||$ is the translation distance and $\bm{\bar{v}}=\bm{v}/q$ is the unit vector along the translation direction. 
In summary, a twist $\bm{\xi}$ can represent three types of motions:
\begin{itemize}
	\item for a rotation, $\bm{\omega} \neq \bm{0}$, $h=0$, and $q=||\bm{\omega}||$ is the rotation angle;
	\item for a translation, $\bm{\omega} = \bm{0}$, $h=\infty$, $\bm{v}=h\bm{\omega}\neq \bm{\infty} $, and $q=||\bm{v}||$ is the translation distance; and
	\item for a helical motion, $\bm{\omega} \neq \bm{0}$, $h \neq 0$, $q=||\bm{\omega}||$ is the rotation angle and $hq$ is the translation distance.	
\end{itemize}

A homogeneous transformation in the Lie group $SE(3)$ can be mapped from a twist in the Lie algebra $se(3)$ via an \textit{exponential mapping} $e^{(\cdot)} : se(3) \mapsto SE(3)$, which is given by \cite{selig2005geometric}
\begin{equation}
	\bm{H} = e^{ \bm{\hat{\xi}} }=\bm{I}_4+\bm{\hat{\xi}}+\frac{1}{q^2}(1-\cos q)\bm{\hat{\xi}}^2 + \frac{1}{q^3}(q - \sin q) \bm{\hat{\xi}}^3
\end{equation}
where $q$ is defined as previously.
Simplifications can be made if the type of the motion is known \cite{selig2005geometric}.

One of the good properties of the exponential mapping is the adjunct transformation.
Let $\bm{H} \in SE(3)$ be a homogeneous transformation, the following equation holds as
\begin{equation}
	\bm{H} e^{\bm{\hat{\xi}}} \bm{H}^{-1} = e^{\bm{H} \bm{\hat{\xi}} \bm{H}^{-1}} = e^{\widehat{Ad(\bm{H})\bm{\xi}}} \label{eq:ad}
\end{equation} 
where $Ad(\cdot)$ is called the \textit{adjoint transformation} and given by
\begin{equation}\label{eq:adexpress}
	Ad(\bm{H}) = Ad\left(\left[\begin{matrix}
	\bm{R}&\bm{t}\\
	\bm{0}^T&1
	\end{matrix}\right]\right) = \left[\begin{matrix}
	\bm{R} & \bm{0} \\
	\bm{\hat{t}}\bm{R} & \bm{R}
	\end{matrix}\right]
\end{equation}
where $\bm{\hat{t}} = \left[\begin{smallmatrix}
0 & -t_3 & t_2\\
t_3 & 0 & -t_1\\
-t_2 & t_1 & 0
\end{smallmatrix}\right]$ is the skew-symmetric matrix induced by $\bm{t} = \left[\begin{smallmatrix}
t_1 & t_2 & t_3
\end{smallmatrix}\right]^T$.
This property will be frequently used in the derivation later.

\subsection{General POE model for a serial-link robot}

\begin{figure}[tb]
	\centering
	\includegraphics[scale=.17]{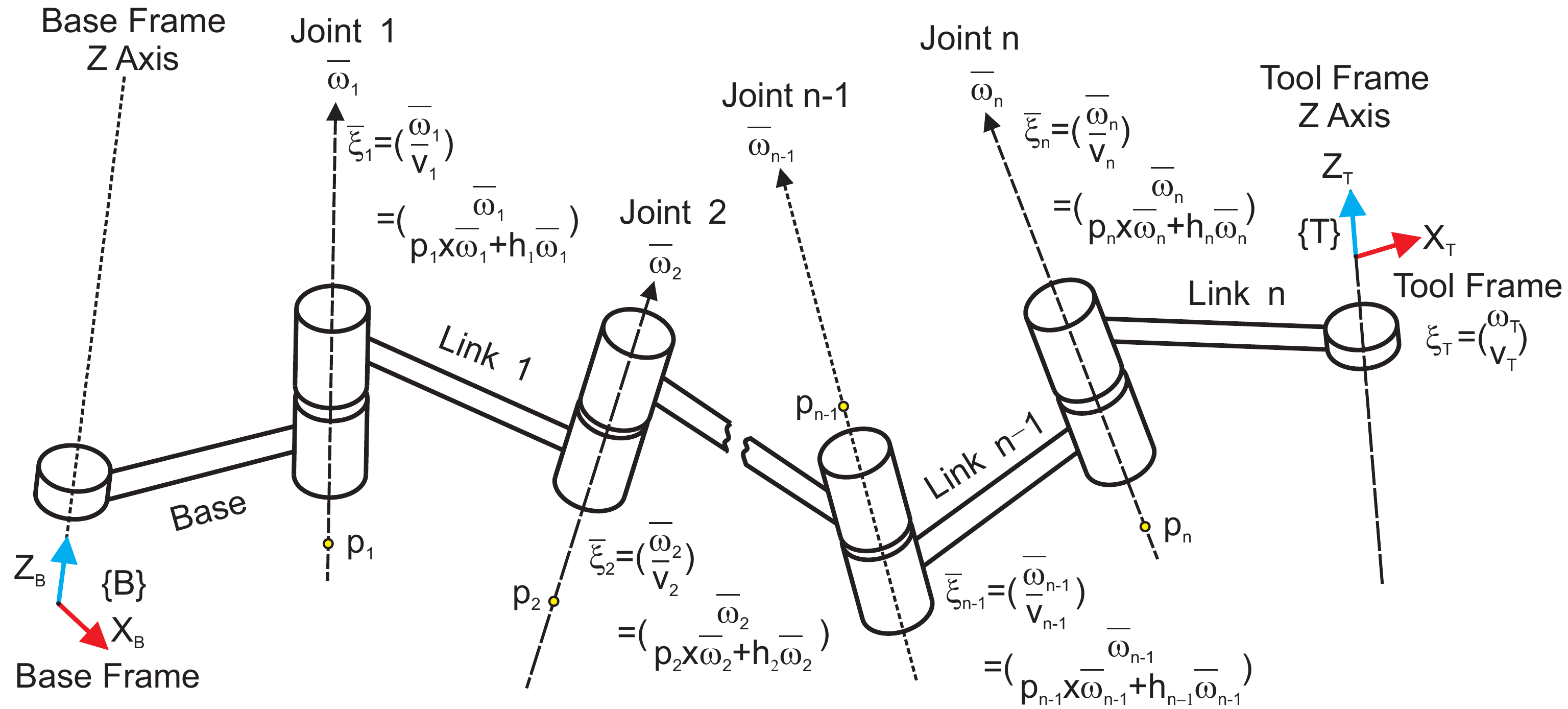}
	\caption{A general POE model for an n-DOF serial-link robot.}
	\label{fig:POEModel}
\end{figure}

Given an $n$-joint serial-link robot, the kinematics can be described by a product of exponentials formula (POE) as
\begin{equation}
	\bm{H} = e^{\bm{\hat{\bar{\xi}}}_1(q_1+\Delta q_1)} e^{\bm{\hat{\bar{\xi}}}_2(q_2+\Delta q_2)} \cdots e^{\bm{\hat{\bar{\xi}}}_n(q_n+\Delta q_n)} e^{\bm{\hat{\xi}}_T}. \label{eq:poe}
\end{equation}
As shown in Fig.~\ref{fig:POEModel}, a base frame \{B\} and a tool frame \{T\} are attached to the base link and the tool link, respectively.
$\Delta q_i$ ($i=1,2,...,n$) is the joint offset, which in most cases is made zero for simplicity.
An initial state is defined when all the joints are in the positions that counteract the offsets, i.e., $q_i = -\Delta q_i$.
In this state, the joint axes are evaluated with respect to the base frame as a series of normalized twists $\bm{\hat{\bar{\xi}}}_i$ ($i=1,2,...,n$), while the transformation from the base frame to the tool frame is described by a general twist $\bm{\hat{\xi}}_T$.

Note that all the three types of joints, namely the revolute joints, the prismatic joints, and the helical joints, can be covered by the general POE model.
The difference between them is only reflected by the difference in each joint twist, as discussed in Section~\ref{sec:special}.

Variations can be made to the POE formula.
For example, if the joint twists are evaluated relative to the tool frame and denoted as $\bm{\hat{\bar{\xi}}}_i^T$, (\ref{eq:poe}) becomes
\begin{equation}
	\bm{H} = e^{\bm{\hat{\xi}}_T} e^{\bm{\hat{\bar{\xi}}}_1^T(q_1+\Delta q_1)} e^{\bm{\hat{\bar{\xi}}}_2^T(q_2+\Delta q_2)}\cdots e^{\bm{\hat{\bar{\xi}}}_n^T(q_n+\Delta q_n)}, 
\end{equation}
which can be referred to as the \textit{tool POE formula} (the original formula (\ref{eq:poe}) can be called the \textit{base POE formula} for contrast).
Or one can establish a local frame on each link, and evaluate the joint twists with respect to each local frame.
Then, (\ref{eq:poe}) changes to the \textit{local POE formula} \cite{chen2001local} as
\begin{equation}
\small
	\bm{H} = \bm{H}_1 e^{\bm{\hat{\bar{\xi}}}_1^{H_1}(q_1+\Delta q_1)} \bm{H}_2 e^{\bm{\hat{\bar{\xi}}}_2^{H_2}(q_2+\Delta q_2)} \cdots \bm{H}_n e^{\bm{\hat{\bar{\xi}}}_n^{H_n}(q_n+\Delta q_n)} \bm{H}_{n+1}
\end{equation}
where $\bm{H}_i$ ($i=1,2,...,n+1$) is the transformation between adjacent local frames.
By using the adjoint transformation, these variations can be proved to be essentially the same as the original model.
Taking the conversion from the local POE formula to the base POE formula for example, we have (omitting $\Delta q_i$ for conciseness)
\begin{eqnarray}
		\bm{H} &=& \bm{H}_1 e^{\bm{\hat{\bar{\xi}}}_1^{H_1}q_1} \bm{H}_2 e^{\bm{\hat{\bar{\xi}}}_2^{H_2}q_2} \cdots \bm{H}_n e^{\bm{\hat{\bar{\xi}}}_n^{H_n}q_n} \bm{H}_{n+1} \nonumber\\
		&\stackrel{(\ref{eq:ad})}{=}& e^{\bm{H}_1\bm{\hat{\bar{\xi}}}_1^{H_1}\bm{H}_1^{-1}q_1}\bm{H}_1\bm{H}_2 e^{\bm{\hat{\bar{\xi}}}_2^{H_2}q_2} \cdots \bm{H}_n e^{\bm{\hat{\bar{\xi}}}_n^{H_n}q_n} \bm{H}_{n+1} \nonumber\\
		&\stackrel{(\ref{eq:ad})}{=}& e^{\widehat{Ad(\bm{H_}1)\bm{\bar{\xi}}}_1^{H_1}q_1}\bm{H}_1\bm{H}_2 e^{\bm{\hat{\bar{\xi}}}_2^{H_2}q_2} \cdots \bm{H}_n e^{\bm{\hat{\bar{\xi}}}_n^{H_n}q_n} \bm{H}_{n+1} \nonumber\\
		&=& e^{\bm{\hat{\bar{\xi}}}_1q_1}\bm{H_1}\bm{H_2} e^{\bm{\hat{\bar{\xi}}}_2^{H_2}q_2} \cdots \bm{H}_n e^{\bm{\hat{\bar{\xi}}}_n^{H_n}q_n} \bm{H}_{n+1} \nonumber\\
		&\vdots&	\nonumber\\		 
		&=& e^{\bm{\hat{\bar{\xi}}}_1q_1} e^{\bm{\hat{\bar{\xi}}}_2q_2} \cdots e^{\bm{\hat{\bar{\xi}}}_nq_n} e^{\bm{\hat{\xi}}_T}.
\end{eqnarray}
Due to the equivalence between these variations, we will just take the base POE formula for analysis in the following parts.

\section{D-H parameterization of the General POE model}

\subsection{D-H parameterization} \label{sec:DH}

\begin{figure}[tb]
	\centering
	
	\subfigure[Revolute joint.]{\label{fig:DHModel_1}\includegraphics[scale=.16]{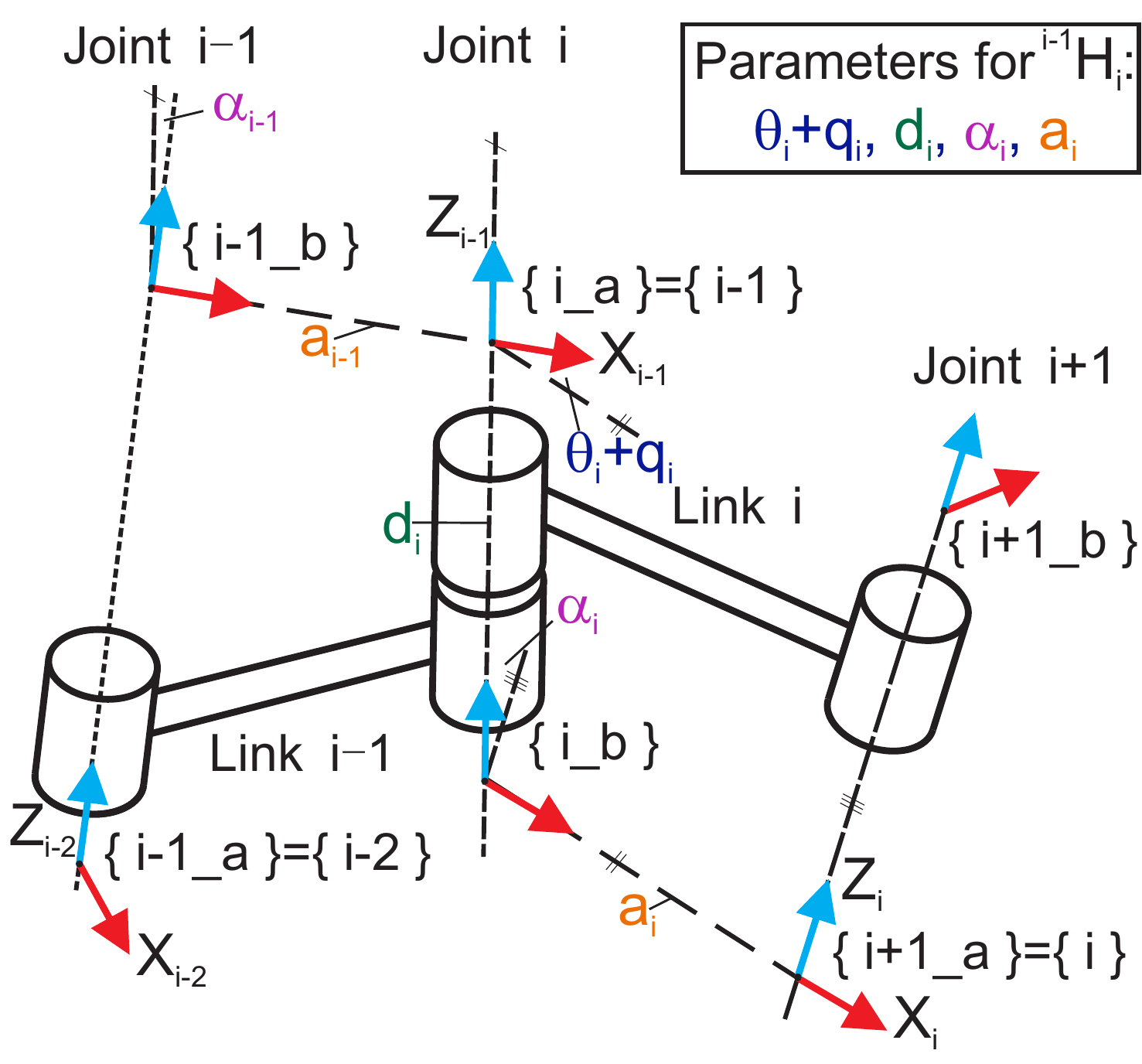}} 
	\subfigure[Prismatic joint.]{\label{fig:DHModel_Prismatic_3}\includegraphics[scale=.16]{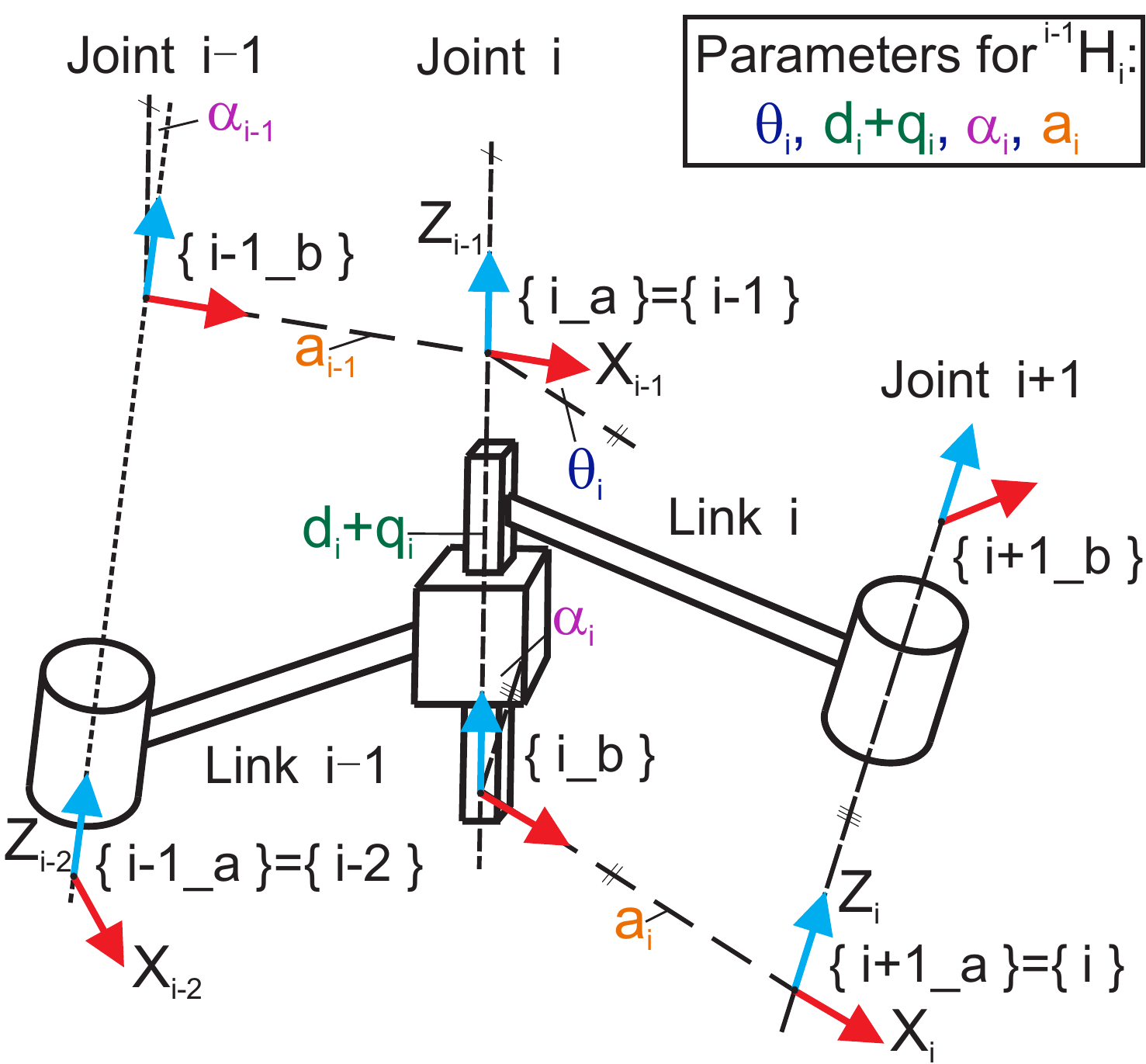}}
	\subfigure[Helical joint.]{\label{fig:DHModel_Helical_1}\includegraphics[scale=.16]{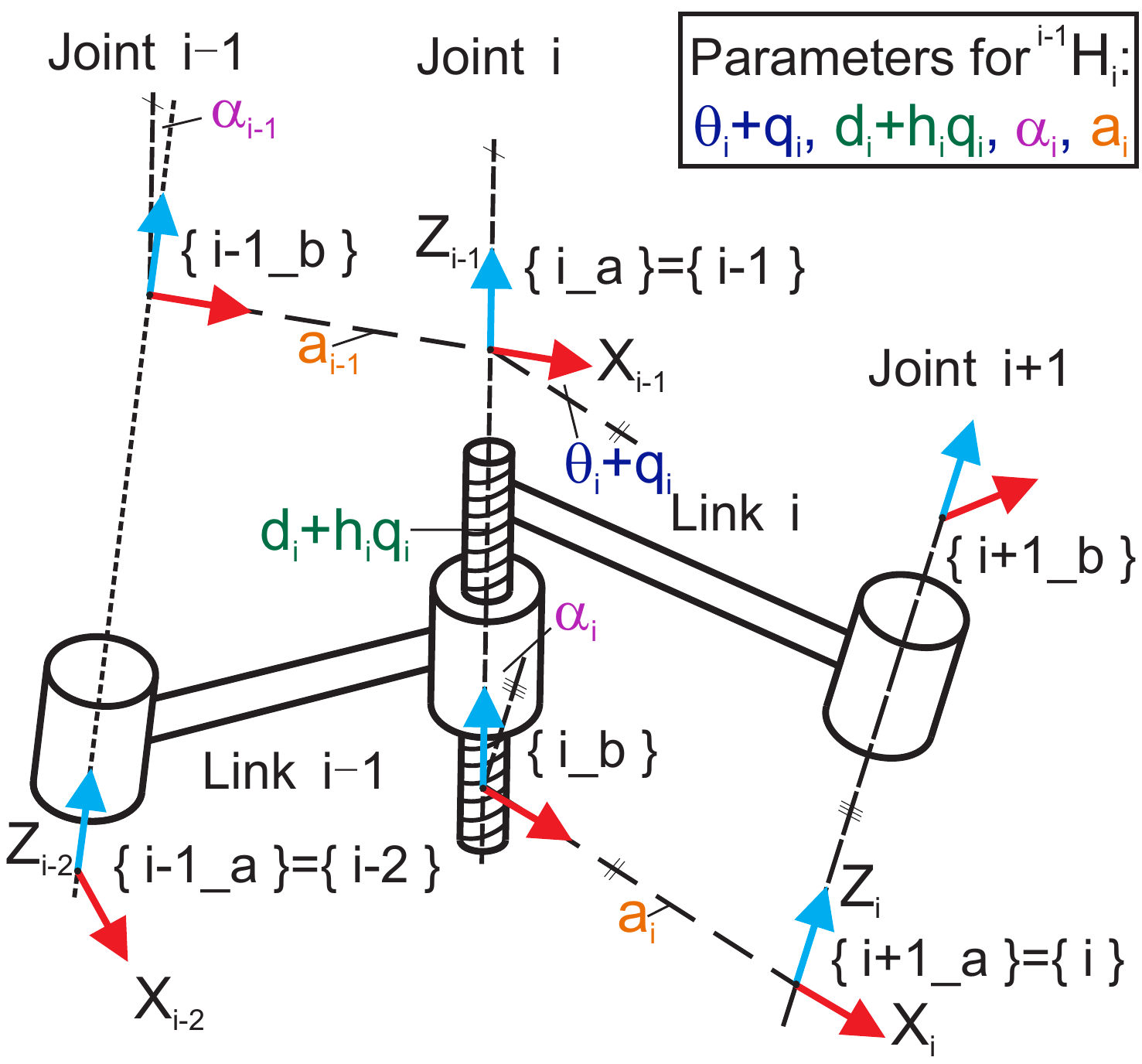}}
	%
	
	
	\caption{Standard D-H parameters for three basic 1-DOF lower pair joints.}
	\label{fig:DHModel}
\end{figure}

\begin{figure}[!tb]
	\centering
	\includegraphics[scale=.17]{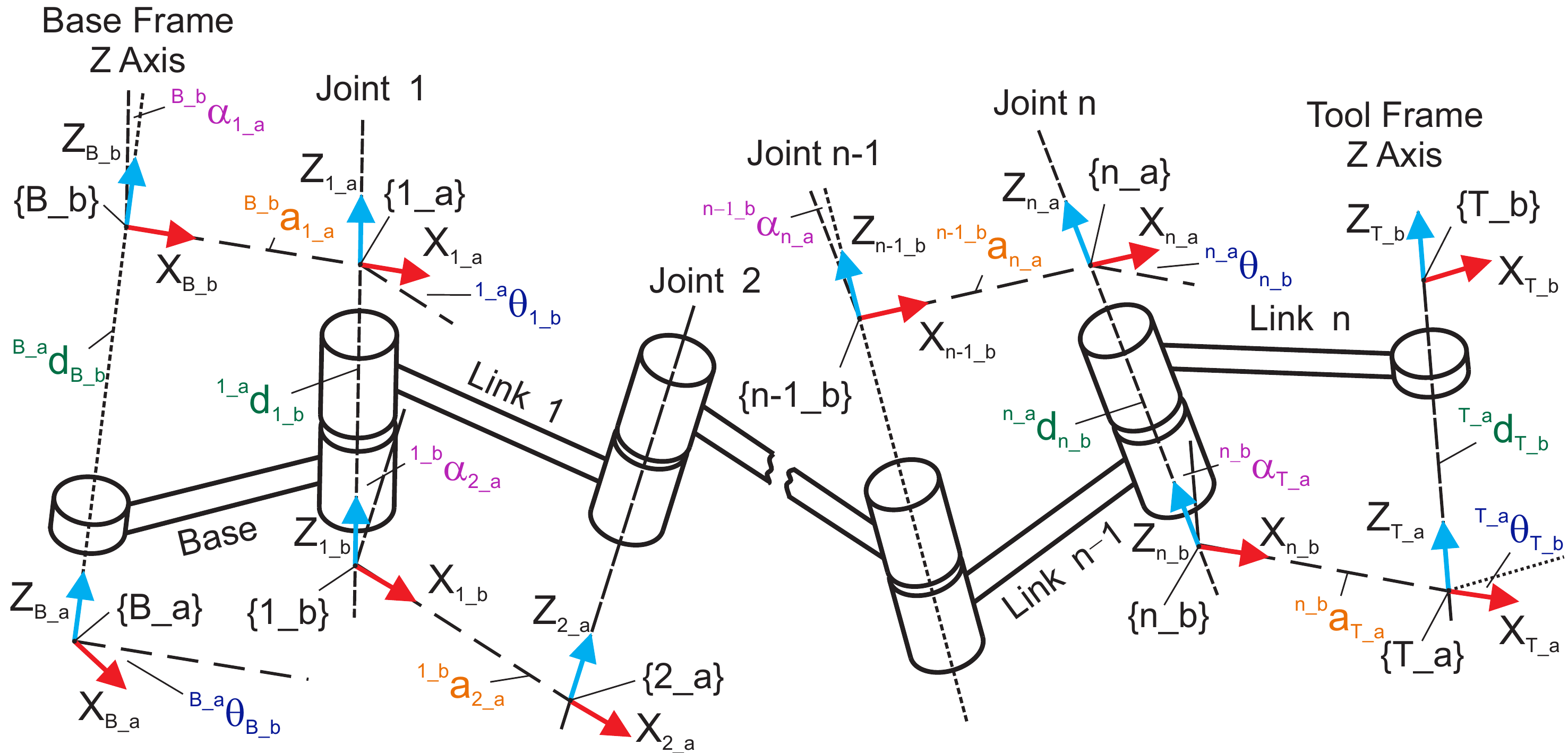}
	\caption{D-H parameters for an n-joint serial-link robot with an arbitrarily located base frame and tool frame.}
	\label{fig:NewDHModel}
\end{figure}

D-H parameterization has been adopted as a standard modeling method by the robotics community as it provides a concise geometric interpretation of the structure of a robot.
The most significant feature of the D-H convention is the use of the common normal between adjacent joint axes, which is a line simultaneously intersecting and perpendicular to the two axes.
As shown in Fig.~\ref{fig:DHModel}, a set of frames are established at the intersections between the joint axes and the common normals.
Four parameters, namely the \textit{joint angle} $\theta$, the \textit{joint offset} $d$, the \textit{link twist} $\alpha$, and the \textit{link length} $a$ are used to define the relative pose between adjacent frames.
It can be verified that the transformation from frame \{ i-1 \} to frame \{ i \} can be constructed by
\begin{equation}
	R_Z(\theta_i+j_iq_i)T_Z(d_i+k_iq_i)R_X(\alpha_i)T_X(a_i)
\end{equation}
where $R_Z(\cdot)$ and $R_X(\cdot)$ mean rotation about $Z$ and $X$, respectively, and $T_Z(\cdot)$ and $T_X(\cdot)$ mean translation along $Z$ and $X$, respectively.
For a revolute joint, $j_i=1$ and $k_i=0$; for a prismatic joint, $j_i=0$ and $k_i=1$; and for a helical joint, $j_i=1$ and $k_i=h_i$.

Note that for some special cases, there could be multiple D-H solutions corresponding to a single configuration.
One such configuration is when adjacent revolute joints have parallel axes, as there are infinite common normals between the axes.
A conventional approach to solve this ambiguity is to let the joint offset $d$ on the first joint be zero for simplicity.
A second such configuration is when a joint is prismatic, as only the direction is meaningful for a translation, while the location of the axis can be anywhere.
In practice, the axis is usually located according to the physical structure of the joint, such as through the center between two edges.

In the POE model, the base frame and the tool frame can be arbitrarily located.
To conform with this rule in the D-H model, the D-H parameterization is applied to the base frame and the tool frame as well, as shown in Fig.~\ref{fig:NewDHModel}.
Under this treatment, the forward kinematics can be expressed as in (\ref{eq:DH}), where $Q_{j_i,k_i}(q_i)=R_Z(j_iq_i)T_Z(k_iq_i)$.
The extraction of $q_i$ holds because of the property that rotations about and translations along the same axis are commutative.

\begin{figure}[!h]
	\footnotesize
	\begin{eqnarray}\label{eq:DH}
	^B\bm{H}_T&=&\underbrace{R_Z(^{B\_a}\theta_{B\_b}) T_Z(^{B\_a}d_{B\_b}) R_X(^{B\_b}\alpha_{1\_a}) T_X(^{B\_b}a_{1\_a})}_{^B\bm{H}_0} Q_{j_1,k_1}(q_1) \nonumber\\
	&&\underbrace{R_Z(^{1\_a}\theta_{1\_b}) T_Z(^{1\_a}d_{1\_b}) R_X(^{1\_b}\alpha_{2\_a}) T_X(^{1\_b}a_{2\_a})}_{^0\bm{H}_1} Q_{j_2,k_2}(q_2) \nonumber\\
	&&\cdots Q_{j_n,k_n}(q_n)\underbrace{R_Z(^{n\_a}\theta_{n\_b}) T_Z(^{n\_a}d_{n\_b}) R_X(^{n\_b}\alpha_{T\_a}) T_X(^{n\_b}a_{T\_a})}_{^{n-1}\bm{H}_n} \nonumber\\ &&\underbrace{R_Z(^{T\_a}\theta_{T\_b}) T_Z(^{T\_a}d_{T\_b})}_{^n\bm{H}_T}.\label{eq:standard_fk}
	\end{eqnarray}
\end{figure}

\subsection{Conversion from general POE Parameters into D-H Parameters} \label{sec:conversion}
To show that a general POE model can be converted into a set of D-H parameters, the following four lemmas are needed.
Among them, Lemma 2-4 have been proved in the authors' previous work\cite{wu2017an}.
Therefore, only the detailed proof of Lemma 1 is given as below.

\begin{lemma}
	For a helical joint, the exponential $e^{\bm{\hat{\bar{\xi}}}q}$ can be converted into the form of $\bm{H}R_Z(q)T_Z(hq)\bm{H^{-1}}$, where $h=\bar{\bm{\omega}}^T\bm{v}$ and  $\bm{H}=R_Z(\theta)T_Z(d)R_X(\alpha)T_X(a)$. 
\end{lemma}

\begin{proof}
	It can be verified that $R_Z(q)T_Z(hq)$ can be written as $e^{\hat{\bar{\bm{\xi}}}^\prime q}$ where $\bar{\bm{\xi}}^\prime = \left[\begin{smallmatrix}
	0 & 0 & 1 & 0 & 0 & h
	\end{smallmatrix}\right]^T$. 
	According to (\ref{eq:ad}), we have
	\begin{equation}
	\bm{H}R_Z(q)T_Z(hq)\bm{H^{-1}} = \bm{H}e^{\hat{\bar{\bm{\xi}}}^\prime q}\bm{H^{-1}} = e^{\bm{H}\hat{\bar{\bm{\xi}}}^\prime\bm{H^{-1}} q} = e^{\widehat{Ad(\bm{H})\bm{\bar{\xi}}^\prime}}q.
	\end{equation}
	
	Then, proving the lemma is equivalent to proving
	\begin{equation} \label{eq:xiAd}
	\bm{\bar{\xi}}=Ad(\bm{H})\bm{\bar{\xi}}^\prime.
	\end{equation} 
	
	Assume $\bm{H}$ can be represented by $R_Z(\theta)T_Z(d)R_X(\alpha)T_X(a)$.
	Expanding $\bm{H}$ and substituting (\ref{eq:adexpress}) into (\ref{eq:xiAd}), we have
	\begin{equation}
	\small
	\bm{\bar{\omega}}\coloneqq\left[\begin{matrix}\label{eq:omega}
	\omega_1\\
	\omega_2\\
	\omega_3
	\end{matrix}\right]=\left[\begin{matrix}
	\sin(\theta)\sin(\alpha)\\
	-\cos(\theta)\sin(\alpha)\\
	\cos(\alpha)
	\end{matrix}\right]
	\end{equation}
	\begin{equation}
	\small
	\bm{v}\coloneqq\left[\begin{matrix}\label{eq:v}
	v_1\\
	v_2\\
	v_3
	\end{matrix}\right]=\left[\begin{matrix}
	\omega_3\sin(\theta)						&	-\omega_2 & \omega_1\\
	-\omega_3\cos(\theta)						&	\omega_1  & \omega_2\\
	\omega_2\cos(\theta)-\omega_1\sin(\theta)	&	0 		  & \omega_3
	\end{matrix}\right]\left[\begin{matrix}
	a\\
	d\\
	h
	\end{matrix}\right].
	\end{equation}
	
	Then, we just need to prove that solutions to (\ref{eq:omega}) and (\ref{eq:v}) always can be found. 
	To this end, two possible cases of $\omega_3$ are examined:
	
	1) $\underline{\omega_3= \pm 1}$ This means $\omega_1=\omega_2=0$ since $\bm{\bar{\omega}}$ is a unit vector.
	It can be verified that ($\theta=\text{atan2}(v_1/\omega_3,-v_2/\omega_3)$, $d=0$, $\alpha=\arccos(\omega_3)$, $a=\sqrt{v_1^2+v_2^2}$,  $h=\bm{\bar{\omega}}^T\bm{v}$) is a solution to (\ref{eq:omega}) and (\ref{eq:v}). 
	Note that $d$ can be actually any value as this corresponds to the configuration with adjacent parallel axes; it is set to be zero for simplicity as discussed in Section~\ref{sec:DH}.
	
	
	2) $\underline{\omega_3\neq \pm 1}$ According to (\ref{eq:omega}), we have
	\begin{eqnarray}
	&\alpha=\pm\arccos(\omega_3) \label{eq:alpha} \\ 
	&\theta=\text{atan2}(\omega_1/\sin(\alpha),-\omega_2/\sin(\alpha)).\label{eq:theta}
	\end{eqnarray}
	Substituting the first two rows of (\ref{eq:omega}) into the last row of (\ref{eq:v}), we have
	\begin{equation}
	a\sin(\alpha)=h\omega_3-v_3.\label{eq:asinalpha}
	\end{equation}
	As the link length $a$ is usually assumed to be nonnegative, the polarity of the right side of (\ref{eq:asinalpha}) can be used to disambiguate the sign of $\alpha$ in (\ref{eq:alpha}).
	
	Then, it can be verified that ($a = \frac{\bm{\bar{\omega}}^T\bm{v}\omega_3-v_3}{\sin(\alpha)}$, $d =\frac{\omega_1v_2-\omega_2v_1}{\omega_1^2+\omega_2^2}$, $h=\bm{\bar{\omega}}^T\bm{v}$)
	is the solution to (\ref{eq:v}).

	In summary, solutions to (\ref{eq:omega}) and (\ref{eq:v}) can always be found and thus the lemma is proved.
\end{proof}

\begin{lemma}
	For a revolute joint, the exponential $e^{\hat{\bm{\bar{\xi}}}q}$ can be converted into the form of $\bm{H}R_Z(q)\bm{H^{-1}}$, where $\bm{H}=R_Z(\theta)T_Z(d)R_X(\alpha)T_X(a)$.
\end{lemma}

\begin{lemma}
	For a prismatic joint, the exponential $e^{\hat{\bm{\bar{\xi}}}q}$ can be converted into the form of $\bm{H}T_Z(q)\bm{H^{-1}}$, where $\bm{H}=R_Z(\theta)T_Z(d)R_X(\alpha)T_X(a)$.
\end{lemma}

\begin{lemma}
	For an arbitrary twist $\bm{\hat{\xi}}$, the exponential $e^{\bm{\hat{\xi}}}$ can be decomposed into a product $e^{\bm{\hat{\xi}}}=\bm{H}_1\bm{H}_2$, where $\bm{H}_1=R_Z(\theta_1)T_Z(d_1)R_X(\alpha_1)T_X(a_1)$ and $\bm{H}_2=R_Z(\theta_2)T_Z(d_2)$.
\end{lemma}

Having these four lemmas, we are able to convert a general POE formula (\ref{eq:poe}) into the D-H parameterization (\ref{eq:DH}).
The process is shown in (\ref{eq:BHT}), with the subscriptions of $Q(\cdot)$ and $\Delta q_i$ omitted for conciseness.
Note that, however, $\Delta q_i$ can be split from $Q(q_i+\Delta q_i)$ and merged into $^{i-1}H_i$.
If joint $i$ is revolute or helical, $\Delta q_i$ can be added to $^{i\_a}\theta_{i\_b}$;
if joint $i$ is prismatic, $\Delta q_i$ can be added to $^{i\_a}d_{i\_b}$.

\begin{figure}[!h]
\scriptsize
\begin{eqnarray}
^B\bm{H}_T&=&e^{\bm{\hat{\bar{\xi}}}_1q_1}e^{\bm{\hat{\bar{\xi}}}_2q_2}\cdots e^{\bm{\hat{\bar{\xi}}}_nq_n} e^{\bm{\hat{\xi}}_T}	\nonumber\\
&\stackrel{\text{Lemma 1-3}}{=}&^B\bm{H}_0 Q(q_1) ^B\bm{H}_0^{-1} e^{\bm{\hat{\bar{\xi}}}_2q_2}\cdots e^{\bm{\hat{\bar{\xi}}}_nq_n} e^{\bm{\hat{\xi}}_T} \nonumber\\
&\stackrel{(\ref{eq:ad})}{=}&^B\bm{H}_0 Q(q_1) e^{^B\bm{H}_0^{-1} \bm{\hat{\bar{\xi}}}_2 {^B\bm{H}_0} {q_2}} {^B\bm{H}_0^{-1}}\cdots e^{\bm{\hat{\bar{\xi}}}_nq_n} e^{\bm{\hat{\xi}}_T} \nonumber\\
&\stackrel{(\ref{eq:ad})}{=}&^B\bm{H}_0 Q(q_1) e^{\bm{\hat{\bar{\xi}}}_2^\prime q_2}{^B\bm{H}_0^{-1}}\cdots e^{\bm{\hat{\bar{\xi}}}_nq_n} e^{\bm{\hat{\xi}}_T}	\nonumber\\
&\stackrel{\text{Lemma 1-3}}{=}&{^B\bm{H}_0} Q(q_1) {^0\bm{H}_1} Q(q_2) {^0\bm{H}_1^{-1}} {^B\bm{H}_0^{-1}} \cdots e^{\bm{\hat{\bar{\xi}}}_nq_n} e^{\bm{\hat{\xi}}_T} \nonumber\\
&\vdots&	\nonumber\\
&=&{^B\bm{H}_0 Q(q_1)} {^0\bm{H}_1} Q(q_2) \cdots {^{n-2}\bm{H}_{n-1}} Q(q_n) {^{n-2}\bm{H}_{n-1}^{-1}} \nonumber\\
&\quad&\quad\quad\cdots {^0\bm{H}_1^{-1}} {^B\bm{H}_0^{-1}} e^{\bm{\hat{\xi}}_T}	\nonumber\\
&=&{^B\bm{H}_0 Q(q_1)} {^0\bm{H}_1} Q(q_2) \cdots {^{n-2}\bm{H}_{n-1}} Q(q_n) e^{\bm{\hat{\xi}}_T^\prime}	\nonumber\\
&\stackrel{\text{Lemma 4}}{=}& {^B\bm{H}_0 Q(q_1)} {^0\bm{H}_1} Q(q_2) \cdots {^{n-2}\bm{H}_{n-1}} Q(q_n) {^{n-1}\bm{H}_{n}} {^{n}\bm{H}_{T}}
\label{eq:BHT}
\end{eqnarray}
\end{figure}

\section{Algorithm Validation} \label{sec:validation}

We validate the proposed method through a PUMA 560 robot whose nominal and actual POE parameters were given in \cite{okamura1996kinematic,he2010kinematic}.
As listed in Table~\ref{tab:POE}, the actual parameters differ a little bit from the nominal ones due to errors in manufacturing, assembly, etc.
By using the algorithm developed in Section~\ref{sec:conversion} (an implementation of the algorithm in MATLAB can be found in the supplementary files), the D-H parameterization from both the nominal and the actual POE parameters is obtained and shown in Table~\ref{tab:DH}.
The parameters that changed significantly between the nominal and the actual are marked to facilitate the discussion in Section~\ref{sec:discussion}.

Note that, however, the actual POE parameters given in \cite{okamura1996kinematic,he2010kinematic} are not normalized.
Therefore, a normalization is performed before the conversion happens, and is given by $\bm{\xi} = \bm{\bar{\xi}}\bar{q}$ where $\bm{\xi}$ is the original twist, $\bm{\bar{\xi}}$ is the normalized twist, and $\bar{q}$ is the normalization factor defined similarly as $q$ in (\ref{eq:xi}).

\begin{table}[!tbp]
	\caption{Parameters of the PUMA 560 robot.}
	\label{tab:POE}
	\centering
	\footnotesize
	\begin{tabular}{cll}
		\toprule
		& \multicolumn{1}{c}{Nominal} & \multicolumn{1}{c}{Actual} \\
		\midrule
		${{\pmb{\xi }}_1}$ & ${\left[ {\begin{smallmatrix}
				{0}&{0}&{1}&{0}&{0}&{0}
				\end{smallmatrix}} \right]^T}$& ${\left[ {\begin{smallmatrix}
				0.04&-0.02&0.999&0.02&0.04&0
				\end{smallmatrix}} \right]^T}$  \\
		${{\pmb{\xi }}_2}$ & ${\left[ {\begin{smallmatrix}
				0&{-1}&0&{ 0}&0&{0}
				\end{smallmatrix}} \right]^T}$& ${\left[ {\begin{smallmatrix}
				0&-1.00002&0&-0.02&0&0.05
				\end{smallmatrix}} \right]^T}$  \\
		${{\pmb{\xi }}_3}$ & ${\left[ {\begin{smallmatrix}
				0 & -1 & 0 & 0 & 0 & -100
				\end{smallmatrix}} \right]^T}$ & ${\left[ {\begin{smallmatrix}
				0.178&-0.984&-0.001&-0.07&0.009&{ - 101}
				\end{smallmatrix}} \right]^T}$  \\
		${{\pmb{\xi }}_4}$ & ${\left[ {\begin{smallmatrix}
				{0}&{0}&{ -1}&{ -50}&{250}&{0}
				\end{smallmatrix}} \right]^T}$& ${\left[ {\begin{smallmatrix}
				0.062&0.013&{ - 0.998}&-51&{249}&0.0752
				\end{smallmatrix}} \right]^T}$  \\
		${{\pmb{\xi }}_5}$ & ${\left[ {\begin{smallmatrix}
				{0}&{-1}&{  0}&{ - 20}&{0}&{-250}
				\end{smallmatrix}} \right]^T}$ & ${\left[ {\begin{smallmatrix}
				0.001&-1.00004&{0}&-20.6&{-0.0206}&-249
				\end{smallmatrix}} \right]^T}$ \\
		${{\pmb{\xi }}_6}$ & ${\left[ {\begin{smallmatrix}
				{0}&{0}&{ -1}&{ - 50}&{250}&{0}
				\end{smallmatrix}} \right]^T}$ & ${\left[ {\begin{smallmatrix}
				0.095&0.031&{ - 0.995}&-51&{249}&0
				\end{smallmatrix}} \right]^T}$ \\
		${{\pmb{\xi }}_{T}}$ & ${\left[ {\begin{smallmatrix}
				{0}&{0}&{ 0}&{ 250}&{50}&{-20}
				\end{smallmatrix}} \right]^T}$ & ${\left[ {\begin{smallmatrix}
				0.02&-0.01&{ 0.01}&249&{51}&-20.6
				\end{smallmatrix}} \right]^T}$ \\
		\bottomrule
	\end{tabular}

\end{table}

\begin{table*}[!tbp]
	\caption{D-H Parameterization of the PUMA 560 Robot Using the Proposed Algorithm}
	\label{tab:DH}
	\centering
	\scriptsize
	\begin{tabular}{cccccccccccccc}
		\toprule
		\multirow{2}{*}{ } & \multicolumn{6}{c}{Nominal} &&\multicolumn{6}{c}{Actual}\\
		\cline{2-7} \cline{9-14}
		& $\theta (^{\circ})$ & $d(mm)$ & $\alpha (^{\circ})$ & $a (mm)$ & $h (mm/^{\circ})$ & $\bar{q}$ && $\theta (^{\circ})$ & $d (mm)$ & $\alpha (^{\circ})$ & $a (mm)$ & $h (mm/^{\circ})$ & $\bar{q}$\\
		\midrule
		$^B\bm{H}_0$ &\color{red}{\underline{0}}&0&0&0& & && \color{red}{\underline{63.4349}} & 1.0000 & 2.5632 & 0 &   &  \\
		$^0\bm{H}_1$ &\color{red}{\underline{0}}&0&90&0&0&1&& \color{red}{\underline{116.5421}} & -1.0214 & -88.8540 & 0.0092 & 0 & 1.0000005\\
		$^1\bm{H}_2$ &\color{red}{\underline{0}}&\color{red}{\underline{0}}&0&\color{red}{\underline{100}}&0&1&& \color{red}{\underline{-88.0290}} & \color{red}{\underline{-567.6554}} & 10.2538 & \color{red}{\underline{0.5114}} & 0 & 1.00002\\
		$^2\bm{H}_3$ &\color{red}{\underline{0}}&\color{red}{\underline{-50}}&90&150&0&1&& \color{red}{\underline{-86.0527}} & \color{red}{\underline{552.7518}} & 90.0434 & 153.3035 & 0.0014 &0.99997\\
		$^3\bm{H}_4$ &180&20&90&0&0&1&& -169.8271 & 30.1922 & 90.7413 & 1.8365 & -0.0000009 &1.000008\\
		$^4\bm{H}_5$ &180&0&90&0&0&1&& -178.0999 & -0.7927 & 91.7710 & 3.1761 & 0.00000001 &1.00004\\
		$^5\bm{H}_6$ &\color{red}{\underline{0}}&\color{red}{\underline{0}}&180&0&0&1&& \color{red}{\underline{82.8018}} & \color{red}{\underline{-37.0573}} & 175.0748 & -0.7202 & 0.0502 &1.000005\\
		$^6\bm{H}_T$ &\color{red}{\underline{0}}&\color{red}{\underline{0}}& & & & && \color{red}{\underline{83.1872}} & \color{red}{\underline{-35.4328}} &   &   &   &  \\
		\bottomrule
	\end{tabular}
\end{table*}


To test the accuracy of the conversion, 100 groups of independent configurations are generated.
In each configuration, the joint variables are uniformly sampled from the interval $[-\pi,\pi]$.
Then we calculate the forward kinematics using both the POE model and the converted D-H model and compare the resultant orientations and positions of the tool frame with respect to the base frame.
The rotation and translation errors are defined as
\begin{eqnarray}
&e_R=||EA(\bm{R}_{DH}^{-1}\bm{R}_{POE})||\\
&e_t=||\bm{t}_{DH}-\bm{t}_{POE}||
\end{eqnarray}
where $\bm{R}_{DH}$ and $\bm{t}_{DH}$ are the rotation matrix and translation vector of the tool frame using the D-H model, and $\bm{R}_{POE}$ and $\bm{t}_{POE}$ are those of the POE model.
In addition, $EA(\cdot)$ means the Euler angle decomposition of a rotation matrix in the order of ZYX.
The errors are plotted in Fig.~\ref{fig:Result}.

It can be seen from Fig.~\ref{fig:Result} that the errors between the D-H model and the POE model are in the order of $10^{-15}$ for rotation and translation.
Considering the round-off error in MATLAB is around $10^{-16}$ by default, which is comparable to the errors of the results if taking the accumulation of errors into account, the two models can be regarded as equivalent for both the nominal and the actual parameters. 

\begin{figure}[!tbp]
	\centering
	\includegraphics[scale=.58]{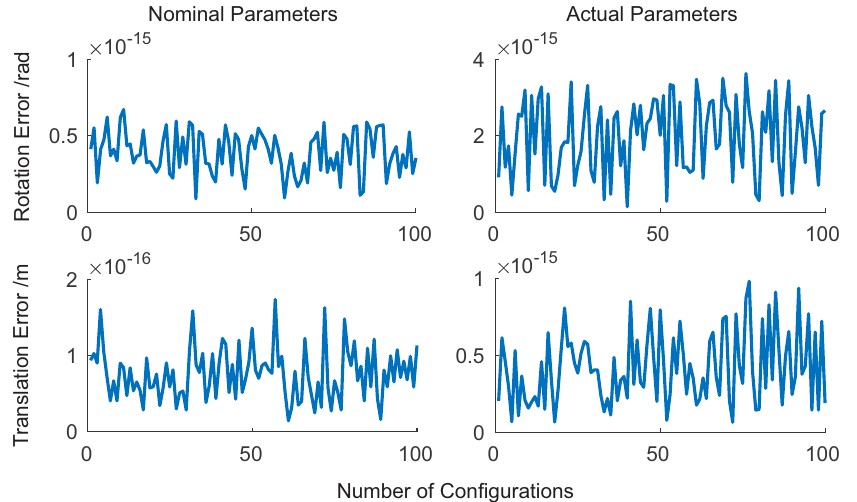}
	\caption{Rotational and translational errors between the normalized POE model and the converted D-H model with 100 randomly selected configurations.}
	\label{fig:Result}
\end{figure}

\section{Discussion on the Identifiability of Kinematic Parameters} \label{sec:discussion}


\subsection{Identifiability of the POE Model}

Recently, debates have been raised around the identifiability of the kinematic parameters of a serial-link robot using the POE formula \cite{li2016poe,chen2014determination,yang2014minimal,he2010kinematic}.
In \cite{he2010kinematic}, it was claimed that the maximum number of the identifiable parameters in a POE formula is 
\begin{equation} \label{eq:c1}
	C_1 = 6r+3t+6
\end{equation}
where $r$ and $t$ mean the number of revolute and prismatic joints, respectively.
In the local POE model proposed in \cite{chen2001local}, the number of parameters to be identified was also as described in (\ref{eq:c1}).
However, this equation does not conform with the conventional opinion that the maximum number of the identifiable parameters should be 
\begin{equation} \label{eq:c2}
	C_2 = 4r+2t+6.
\end{equation}
Later, it was pointed out that the maximum number of identifiable parameters does coincide with (\ref{eq:c2}) both in the standard POE formula \cite{li2016poe,yang2014minimal} and in the local POE formula \cite{chen2014determination}.
We will show that it is quite straightforward to analyze the identifiability with the assistance of the geometric interpretation of the D-H parameterization.

First, it is easy to see that the joint offsets $\Delta q_i$ cannot be identified together with the initial twist $\bm{\xi}_T$.
As mentioned previously, the initial configuration is defined when all the joints are in the positions that counteract the joint offsets.
If $\bm{\xi}_T$ is included in the identification model, then there are no constraints on the initial configuration.
As $\Delta q_i$ can be merged with $\theta_i$ or $d_i$ after D-H parameterization, it is impossible to uniquely distinguish the joint offsets from the D-H parameters. 

Second, it can be seen that a comprehensive decomposition of the maximum number of identifiable parameters should be expressed as
\begin{equation} \label{eq:c3}
	C_3 = \underbrace{5h + 4r + 2t}_{\bm{\bar{\xi}_i}} + \underbrace{n}_{\bar{q}_i} + \underbrace{6}_{\bm{\xi}_{T}}
\end{equation}
where $h$, $r$, $t$, and $n$ stand for the number of helical, revolute, prismatic, and general joints.
Eq. (\ref{eq:c3}) has three components: the normalized joint twists that take all the three basic joint types into account; the normalization factors that can describe the linear variation of the joint variables, such as the coefficient between the actual joint angles and the encoder readings; and the initial placement of the tool frame with respect to the base frame.
In a conventional setting, the helical motions and the normalization factors are excluded from the identification model, and thus (\ref{eq:c3}) degenerates to (\ref{eq:c2}).
In this case, a normalization and an orthogonalization are required in each identification step, or the two constraints need to be embedded in the identification model.
And (\ref{eq:c3}) evolves to (\ref{eq:c1}) if assuming that all the revolute joints have a pattern of helical motion and all the joint variables have linear variation.
Based on this analysis, it is clear that different set of parameters should be selected according to the practical characteristics of the robot to be identified.

\subsection{Identifiability of the D-H Model}
Due to the equivalence between the POE model and the D-H model, the number of the identifiable parameters should be the same.
Since the POE based identification model has included the complete set of identifiable parameters, it should also be the case with the converted D-H model.
This gives some indication on the identifiability of the base frame and the tool frame in the D-H model.
In the traditional analysis of the identifiability of the D-H model as represented in (\ref{eq:c2}), it is simply regarded that the extra six parameters come from the placement of the tool frame because there are six degrees of freedom to restrict a general rigid-body motion.
Based on the new D-H parameterization proposed in this paper as shown in Fig.~\ref{fig:NewDHModel}, it can be seen that among these six parameters, four are allocated to the transform from the base frame to the first joint frame, and two are allocated to the transform from the last joint frame to the tool frame.
This means that the extra six identifiable parameters are actually restricted rather than arbitrarily allocated.
In the conventional D-H parameterization, the base frame is defined to coincide with the first joint frame.
In that case, there are only two extra parameters that can be identified, not six.
This has not been correctly pointed out to the best of the authors' knowledge.

The singularity of the D-H based identification model when the robot has adjacent parallel joint axes is also reflected in the example given in Section~\ref{sec:validation}.
In Table~\ref{tab:DH}, the parameters that changed significantly ($> 60^\circ$ or $> 30$ mm) between the nominal and the actual are marked red and underlined.
In the example model, the axis pairs (Base Z, joint 1), (joint 2, joint 3), and (joint 6, Tool Z) are parallel.
As a result, some of the identified D-H parameters in ($^B\bm{H}_0$, $^0\bm{H}_1$), ($^1\bm{H}_2$, $^2\bm{H}_3$), and ($^5\bm{H}_6$, $^6\bm{H}_T$) changed significantly from their nominal values.
This singularity problem has been previously pointed out from the geometric viewpoint that the D-H parameterization has different modeling rules for parallel and nonparallel axis pairs.
From the derivation in this paper, the mathematical reason for this problem can be found as the discontinuity between the two cases for $\omega_3 = \pm 1$ and $\omega_3 \neq \pm 1$.
Due to the singularity, the robustness of the identification algorithm directly based on the D-H parameters is affected.
However, the results in this paper suggest that a solution to circumvent the singularity could be to convert the D-H parameters into POE parameters first, identify the POE parameters, and then convert the POE parameters back to D-H parameters for compensation.


\section{Conclusion}
This paper has developed an analytic solution to automatically convert a POE model into a D-H model for a robot with revolute, prismatic, and helical joints, which are the complete set of three basic one DoF lower pair joints for the construction of a serial-link robot.
The algorithm has been validated by an example on the PUMA 560 robot.
An implementation of the algorithm in MATLAB is also provided.

The identifiability of the kinematic parameters is analyzed based on the equivalence of the two models.
It is found that the maximum number of identifiable parameters in a general POE model is $5h+4r+2t+n+6$ where $h$, $r$, $t$, and $n$ stand for the number of helical, revolute, prismatic, and general joints, respectively. 
It is also suggested that the identifiability of the base frame and the tool frame in the D-H model is restricted rather than the arbitrary six parameters as assumed previously.
In addition, the results indicate that the singularity problem in the D-H based calibration can be circumvented by converting the D-H model to the POE model for calibration and then converting the POE model back to the D-H model for compensation.

\bibliography{DHPOE.bib}
\bibliographystyle{IEEEtran}

\end{document}